\documentclass{article}

\usepackage[preprint,nonatbib]{neurips_2022}

\usepackage[utf8]{inputenc} 
\usepackage[T1]{fontenc}    
\usepackage{hyperref}       
\usepackage{url}            
\usepackage{booktabs}       
\usepackage{amsfonts}       
\usepackage{nicefrac}       
\usepackage{microtype}      
\usepackage{xcolor}         
\usepackage[pdftex]{graphicx}

\usepackage{amsmath}
\usepackage{amssymb}
\usepackage{mathtools}
\usepackage{amsthm}

\DeclareMathOperator*{\argmax}{argmax}
\DeclareMathOperator*{\defeq}{\coloneqq}
\DeclareMathOperator*{\expected}{\mathbb{E}}

\theoremstyle{plain}
\newtheorem{theorem}{Theorem}[section]

\newtheorem{lemma}[theorem]{Lemma}

\theoremstyle{definition}
\newtheorem{definition}[theorem]{Definition}

\theoremstyle{remark}

\title{Generalization Analysis on \\Learning with a Concurrent Verifier}

\author{%
  Masaaki Nishino, ~~~ Kengo Nakamura, ~~~ Norihito Yasuda \\
  NTT Communication Science Laboratories, NTT Corporation\\
  \texttt{\{masaaki.nishino.uh, kengo.nakamura.dx, norihito.yasuda.hn\}@hco.ntt.co.jp} 
}

\begin{document}

\maketitle








\begin{abstract}

Machine learning technologies have been used in a wide range of practical systems.
In practical situations, it is natural to expect the input-output pairs of a machine learning model to satisfy some requirements.
However, it is difficult to obtain a model that satisfies requirements by just learning from examples.
A simple solution is to add a module that checks whether the input-output pairs meet the requirements and then modifies the model's outputs.
Such a module, which we call a {\em concurrent verifier} (CV), can give a certification, although how the generalizability of the machine learning model changes using a CV is unclear. 
This paper gives a generalization analysis of learning with a CV. We analyze how the learnability of a machine learning model changes with a CV and show a condition where we can obtain a guaranteed hypothesis using a verifier only in the inference time.
We also show that typical error bounds based on Rademacher complexity will be no larger than that of the original model when using a CV in multi-class classification and structured prediction settings. 
\end{abstract}

\section{Introduction}
\label{seq:introduction}

As machine learning technology matures, many systems have been developed that exploit machine learning models. 
When developing a system that uses a machine learning model, a model with merely small prediction error is not satisfactory due to real-field requirements. For example, an object recognition model that is
sensitive to slight noise would cause security issues~\cite{NEURIPS2018_be53d253-linear-network-verificaton,tjeng2018evaluating-mipverify}, or a model 
with unexpected output would increase a system's cost for dealing with it.
Thus, we want the input-output pairs of a machine learning model to satisfy some {\em requirements}. 
However, it is difficult to obtain a model that satisfies the requirements by just learning from examples. Moreover, 
since the learned models tend to be complex and
the input domain tends to be quite large,  it is unrealistic to certify that every input-output pair satisfies the requirements. 
In addition, even if we find an input-output pair that does not satisfy the requirements, modifying a model is difficult since we have to re-estimate it from the training examples.

This paper considers a way to obtain a machine learning model whose input-output pairs satisfy the required properties.  We address the following assumptions for a situation where a machine learning model is used. First, we can judge whether input-output pair $(x, h(x))$ satisfies the requirements, where $h: \mathcal{X} \to \mathcal{Y}$ is a machine learning model or a hypothesis. As we show below, important use cases fit
this setting. Second, a machine learning model already exists whose prediction error is small enough, although its input-output pairs are not guaranteed
to satisfy the requirements. This second assumption is also reasonable since modern machine learning models show
sufficient prediction accuracy in various tasks.
Under these assumptions, a practical choice for addressing this problem isn’t changing the machine learning model but
adding a module that checks the input-output pairs of machine learning model $h$. 
We call this module {\em a concurrent verifier} (CV). Fig.~\ref{fig:concept} shows the system configuration of a machine learning model with a CV.
The verifier checks whether the input-output pair $(x, h(x))$ satisfies the required properties. If it satisfies the requirements, it outputs $h(x)$. If not, then it rejects $h(x)$ and modifies or requests the learning model to modify its output. A machine learning model and verifier pair can be seen as another machine learning model whose
input-output pairs are guaranteed to satisfy the required conditions. 

Although a model with a verifier can guarantee that its input-output pairs satisfy requirements, its effect on prediction error is unclear.
This paper gives theoretical analyses of the generalization errors of a machine learning model with a CV. We focus on how the learnability of the original model, denoted as hypothesis class $\mathcal{H}$, can change by using the verifier. First, we consider a situation where we use a CV only in the inference phase. This setting corresponds to a case where the required properties are unknown when we are in the training phase. If the hypothesis class is PAC-learnable,  
we can obtain a guaranteed hypothesis using a verifier only in the inference time. 

Second, we consider a situation where we know the requirements when learning the model. This situation corresponds to viewing the learnability of
hypothesis set $\mathcal{H}_c$, which is obtained by modifying every hypothesis $h \in \mathcal{H}$ to satisfy the requirements. Hence we compare
the generalization error upper bounds of $\mathcal{H}_c$ with those of $\mathcal{H}$. On the multi-class classification setting, we show that
existing error bounds~\cite{NEURIPS2018_1141938b-multiclass,10.5555/2371238-mohri} based on the Rademacher complexity of $\mathcal{H}$  are also bounds  of modified hypothesis $\mathcal{H}_c$  for any
input-output requirements. Moreover, we give similar analyses for a structured prediction task, which 
is a
kind of multi-class classification where set of classes $\mathcal{Y}$ can be decomposed into substructures. It is worth analyzing
the task since many works address the constraints in structured prediction. 
Some works give error bounds for structured prediction tasks, which are tighter than 
simply applying the bound for multi-class classification tasks~\cite{NEURIPS2021_structure,NIPS2016_535ab766-structure,ijcai2021-0391-structure}.
Similar to the case of multi-class classification, we
show that existing Rademacher complexity-based bounds for the structured prediction of $\mathcal{H}$  are also the bounds for $\mathcal{H}_c$. 
\begin{figure}[t]
\centering
\includegraphics[width=0.6\hsize, keepaspectratio]{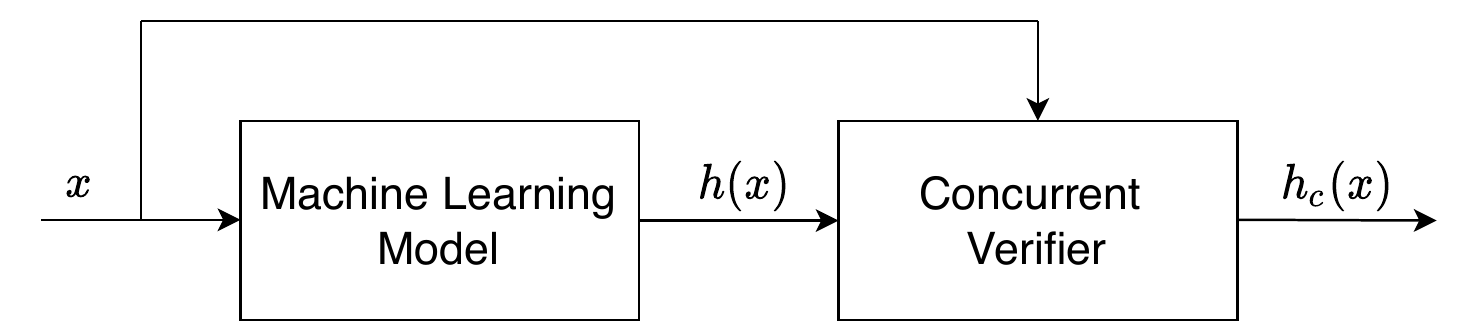}
\caption{ Overview of a machine learning model with a concurrent verifier that 
checks whether input-output pairs of a model satisfy requirements.}
\label{fig:concept}
\end{figure}

Our main contributions are as follows: a) We introduce a concurrent verifier, which is a model-agnostic way to guarantee that machine learning models satisfy the required properties. Although a similar mechanism was used in some existing models, our model gives a generalization analysis that does not depend on a specific model.
    b) We show that if hypothesis class $\mathcal{H}$ is PAC-learnable, then using a verifier at the inference time can give a hypothesis with a guarantee in its generalization error. Interestingly, if H is not PAC-learnable, we might fail to obtain a guaranteed hypothesis even if the requirements are consistent with distribution $\mathcal{D}$. 
    c) We show that if we use a CV in a learning phase of multi-class classification tasks, then the theoretical error bounds of $\mathcal{H}$ based on the Rademacher complexity will not increase with any input-output requirements. We also give similar results for  structured prediction tasks.

\subsection{Use Cases of a Concurrent Verifier}

The following are some typical use cases for CVs. 

\textbf{Error-sensitive applications:} A typical situation where we want to use a verifier is 
that some prediction errors might cause severe effects, which we want to avoid. For example, a recommender system might limit the 
set of candidate items depending on user attributes. Although such a rule might degrade the prediction accuracy, practically a safer model is preferable.

\textbf{Controlling outputs of structured prediction:}
Constraints are frequently used in structured prediction tasks for improving the performance or
the controllability of the outputs.
For example, some works \cite{roth2005integer,10.1007/s10994-012-5296-5-structured}  exploited the constraints on sequence labeling tasks for  reflecting background knowledge to 
improve the prediction results. 
More recently, some works \cite{hokamp-liu-2017-lexically,anderson-etal-2017-guided} exploited the constraints in language generation tasks, including image captioning and machine translation,
and restricted a model to output a sentence that includes given keywords. Since the constraints
used in this previous work can be written as a logical formula,
our CV model can represent them as requirements.

\textbf{Robustness against input perturbations:}
If a machine learning model changes its output because we modified its input from $x$ to $x^\prime$, which is very close to $x$, then the model is described as sensitive against a small change~\cite{42503intriguring-adversarial}. It might be a security risk if a model is sensitive since its behavior is unpredictable. Therefore, some methods evaluate and verify
the robustness of neural networks against small perturbations~\cite{tjeng2018evaluating-mipverify,NEURIPS2018_be53d253-linear-network-verificaton}. Existing verification methods check
a machine learning model's robustness around input $x$ by determining whether $x^\prime$ exists that is close to $x$ and whether model $f$ gives different outputs, i.e., $h(x) \neq h(x^\prime)$, for verification 
samples $x_1, \ldots, x_n$.
Although these verification methods can test a model, they do not directly show how to obtain
a robust model.

A CV can fix a model to achieve robustness around samples $x_1, \ldots, x_n$ by
 setting a rule of form: ``$h(x^\prime)$ must equal $h(x_i)$ if $x^\prime$ is close to $x_i$.'' Although this solution might not guarantee robustness where samples are scarce, adding enough non-labeled verification samples is often a reasonable choice. 

\section{Related Work}

Machine learning models that can exploit constraints have been investigated in many research fields, including statistical symbolic learning and structured prediction. For example, Markov logic networks~\cite{richardson2006markov}, Problogs~\cite{10.5555/1625275.1625673-problog}, and probabilistic circuit models~\cite{KR148005-psdd} integrate statistical models with symbolic logic formulations. Since these models can incorporate hard constraints represented by symbolic logic, they can guarantee input-output pairs. 
However, previous research focused on their practical performance and gave little theoretical analysis of their learnability when hard constraints are used. 
Moreover, previous works integrated the ability to exploit constraints into specific models. 
In contrast, our CV is model-agnostic and can be used in combination with a wide range of machine learning models.

Recently, the verification of machine learning models has been gathering more attention. Attempts have verfified whether a machine learning model has the desired properties \cite{NEURIPS2018_be53d253-linear-network-verificaton,tjeng2018evaluating-mipverify,10.1007/978-3-319-63387-9_5-relplex,NEURIPS2018_f2f44698-verify}. Exact verification methods use 
integer programming (MIP)~\cite{tjeng2018evaluating-mipverify}, constraint satisfaction (SAT)~\cite{Narodytska2020In-bnn-sat}, and a satisfiable module theory (SMT) solver~\cite{10.1007/978-3-319-63387-9_5-relplex} to assess the robustness of a neural network model against input noise. 
These approaches aim to obtain models that fulfill the required properties. However, verification methods
cannot help modify the models if they do not satisfy the requirements. If we want ML models to 
meet requirements, post-processing is needed as our concurrent verification model.

Other methods can give upper bounds on generalization error, including
VC-dimension~\cite{doi:10.1137/1116025vc} and its extensions~\cite{10.5555/2789272.2912074-daniely,10.1023/A:1022605311895-natarajan}, Rademacher complexity~\cite{10.5555/944919.944944-rademacher1,10.2307/2700001-rademacher2}, stability~\cite{10.5555/1756006.1953019-stability}, and PAC-Bayes~\cite{10.1145/279943.279989-pacbayes-origin,Alquier2021UserfriendlyIT-pacbayes}. We use Rademacher
complexity in the following analysis since it is among the most popular tools for giving theoretical upper bounds
on generalization error. Rademacher complexity also has some extensions, including local Rademacher complexity~\cite{NEURIPS2018_1141938b-multiclass}
and factor graph Rademacher complexity~\cite{NIPS2016_535ab766-structure}. We can provide theoretical guarantees on these extended measures.

\section{Preliminaries}

Our notation follows a previous work \cite{10.5555/2621980-understanding}.
We first introduce the notations used in the following sections.
Let $\mathcal{X}$ denote the domain of the inputs, let $\mathcal{Y}$ be the domain of the labels, and let $Z$ be the domain of the examples defined as $Z\defeq \mathcal{X} \times \mathcal{Y}$. Let $\mathcal{H}$ be a hypothesis class, and let $\ell : \mathcal{H} \times Z \to \mathbb{R}_{+}$ be a loss function. 
Training data  $S = (z_1, \ldots, z_m) \in Z^m$ is a finite sequence of size $m$ drawn i.i.d. from a fixed but unknown probability distribution $\mathcal{D}$ on $Z$.
Learning algorithm $A$ maps training data $S$ to hypothesis $h$. We use notation $A(S)$ to denote the hypothesis that learning algorithm $A$ returns upon receiving $S$. We represent set $\{1, \ldots, K\}$ as $[K]$.

Given distribution $\mathcal{D}$ on $Z$, 
we denote by $L_{\mathcal{D}}(h)$ {\em the generalization error} and by $L_S(h)$ {\em the empirical error} of $h$ over $S$, defined by
\begin{equation}
L_{\mathcal{D}}(h) \defeq \mathop{\mathbb{E}}_{z \sim \mathcal{D}}[\ell(h, z)] \,, ~~~~~ L_{S}(h) \defeq \frac{1}{m}\sum_{i=1}^{m}\ell(h, z_i) \,.
\end{equation}\\
\textbf{PAC learnability:} We introduce PAC learnability and agnostic PAC learnability as follows. 
\begin{definition}{(Agnostic PAC learnability)}
Hypothesis class $\mathcal{H}$ is {\em agnostic PAC-learnable} if there exists function $m_{\mathcal{H}} : (0, 1)^2 \to \mathbb{N}$ and learning algorithm $A$ with the following property: For every $\epsilon, \delta \in (0, 1)$ and distribution $\mathcal{D}$ over $Z$, if $S$  consists of  $m \geq m_{\mathcal{H}}(\epsilon, \delta)$ i.i.d. examples generated by $\mathcal{D}$, then with at least probability $1 - \delta$, the following holds:
\begin{align}
\label{eq:apac}
    L_{\mathcal{D}}(A(S)) \leq \min_{h^\prime \in \mathcal{H}}L_{\mathcal{D}}(h^\prime) + \epsilon \,.
\end{align}
\end{definition}
Distribution $\mathcal{D}$ is {\em realizable} by hypothesis set $\mathcal{H}$ if $h^\ast \in \mathcal{H}$ exists such that 
$L_{\mathcal{D}}(h^\ast) = 0$.
If $\mathcal{D}$ is realizable by agnostic PAC-learnable hypothesis $\mathcal{H}$, then $\mathcal{H}$ is {\em PAC-learnable}.
If $\mathcal{H}$ is PAC-learnable, then Eq. (\ref{eq:apac}) becomes $L_{\mathcal{D}}(A(S)) \leq  \epsilon$ since $\min_{h^\prime \in \mathcal{H}} L_\mathcal{D}(h^\prime) = 0$.

\textbf{Rademacher complexity:} 
In the following sections, we use Rademacher complexity for deriving the generalization bounds. 
Given loss function $\ell(h, z)$ and hypothesis class $\mathcal{H}$, we denote $\mathcal{G}$ as
\begin{align*}
    \mathcal{G} \defeq \ell \circ \mathcal{H} \defeq \{z \mapsto \ell(h, z) : h \in \mathcal{H}\}.
\end{align*}

\begin{definition}{(Empirical Rademacher complexity)}
Let $\mathcal{G}$ be a family of functions mapping from $Z$ to $\mathbb{R}$, and let $S = (z_1, \ldots, z_m) \in Z^m$ be the training data of size $m$. Then 
{\em the empirical Rademacher complexity} of $\mathcal{G}$ with respect to $S$ is defined:
\begin{align*}
R_{S}(\mathcal{G}) \defeq \mathop{\mathbb{E}}_{\boldsymbol \sigma}\left[\sup_{g \in \mathcal{G}} \sum_{i=1}^{m} \sigma_i g(z_i) \right] \,,
\end{align*}
where $\boldsymbol{\sigma} = (\sigma_1, \ldots, \sigma_m) \in \{\pm 1\}^m$ are random variables distributed i.i.d. according to 
$\mathbb{P}[\sigma_i = 1] = \mathbb{P}[\sigma_i = -1] = 1/2$. {\em The Rademacher complexity} of $\mathcal{G}$ is defined as the expected value of
the empirical Rademacher complexity:
\begin{align*}
R_{m}(\mathcal{G}) \defeq \mathop{\mathbb{E}}_{S \sim \mathcal{D}^m}[R_S(\mathcal{G})] \, .
\end{align*}
\end{definition}

\section{Concurrent Verifier}
\label{sec:verifier}
Next we give a formal definition of a CV. A CV works with a machine learning model, which is function $h: \mathcal{X} \to \mathcal{Y}$.
If $x$ is given to the model, which outputs $h(x)$, then the verifier checks  whether  $(x, h(x))$ satisfies the required property. 
We assume that the required property can be represented as {\em requirement function} $c: (\mathcal{X}\times \mathcal{Y}) \to \{0, 1\}$. If $c(x, h(x)) = 1$, then the pair satisfies the property; if $c(x, h(x)) = 0$, then it does not. Requirement function $c$ can be represented by a set of deterministic rules. For example, if $\mathcal{X} = \mathbb{R}$ and 
$\mathcal{Y} = \{0, 1\}$, then the requirements can be in the following form: ``if $x > 0$, then $y \neq 0$.'' 
We assume that for all possible input $x \in \mathcal{X}$, there exists $y \in \mathcal{Y}$ such that $c(x, y) = 1$ for avoiding the situation where the requirements
are unsatisfiable for any output $y$. This assumption can be easily relaxed if we allow a machine learning model to reject unsatisfiable input $x$.

After checking the input-output pair, a verifier modifies output $h(x)$ depending on the value of $c(x, h(x))$.
If $c(x, h(x)) = 1$, the verifier outputs $h(x)$ since it satisfies the requirements. If $c(x, h(x)) = 0$, then the verifier modifies $h(x)$ to some $y \in \mathcal{Y}$ that satisfies $c(x, y) = 1$. If we use a verifier with a machine learning model that corresponds to $h$, then the combination of the model and the verifier can be seen as function $h_c : \mathcal{X} \to \mathcal{Y}$, defined as
\begin{align}
\label{eq:modify}
    h_c(x) \defeq \left\{ 
    \begin{array}{ll}
    h(x) & \mbox{ if } c(x, h(x)) = 1 \,\\
    y_c  &\mbox{ if } c(x, h(x)) = 0 \,
    \end{array}    
    \right. \,,
\end{align}
where $y_c \in \mathcal{Y}$ satisfies $c(x, y_c) = 1$ and is selected deterministically. When $\mathcal{Y} = [K]$, an example for selecting
minimum $i \in [K]$  satisfying $c(x, i) = 1$ as $y_c$ is a reasonable choice. When $\mathcal{Y} = [K]$ and $h(x)$ is made by scoring functions $h(x, y) : (\mathcal{X} \times \mathcal{Y}) \to \mathbb{R}$, it is also reasonable to select $y^\ast$ such that  $y^\ast = \argmax_{y \in \mathcal{Y}, c(x, y) = 1} h(x, y)$.
Learning a model corresponds to selecting hypothesis $h$ from hypothesis class $\mathcal{H}$. Therefore, learning a model with a CV corresponds to
choosing a hypothesis from the modified hypothesis class: $\mathcal{H}_c  = \{h_c : h \in \mathcal{H}\}$. By definition, every hypothesis in $\mathcal{H}_c$ satisfies the requirements, and thus we can guarantee that the model satisfies the condition if we select a hypothesis from $\mathcal{H}_c$.
In the following sections, we analyze the learnability of $\mathcal{H}_c$ by comparing it with that of $\mathcal{H}$. 

\section{Inference Time Verification}
\label{seq:ltv}
We first analyze the change of the generalization errors when we use a verifier only in an inference phase. In other words, requirements are unknown in the learning
phase,and we estimate hypothesis $\hat{h} = A(S)$ from hypothesis class $\mathcal{H}$ by using training data $S$ and algorithm $A$.
In the inference phase, we use a CV to modify $\hat{h}$ to $\hat{h}_c$. We call this setting the
{\em inference time verification } (ITV). 
This class of situations contains many exciting settings: 1) pre-trained machine learning models used in a wide range of applications, and 2) models that are hard to replace, which might encounter different requirements from those at the learning time in the long run.

In this section, we give analyses on a multi-class classification setting. We set $\mathcal{Y} = [K]$, and hypothesis class $\mathcal{\mathcal{H}}$ is set of mappings $h:\mathcal{X} \to  [K]$. We also assume that loss function $\ell$ is $\operatorname{0-1}$ loss defined as $\ell_{\operatorname{0-1}}(h, (x,y)) \defeq \mathbf{1}_{h(x) \neq y}$, where $\mathbf{1}$ is an indicator function. 

The following theorem shows a situation where ITV works well: a situation where the generalization error of $\hat{h}_c$ does not exceed that of the
other hypotheses in $\mathcal{H}_c$ with high probability. 
\begin{theorem}
If $\mathcal{Y} = [K]$, and hypothesis class $\mathcal{H}$ is PAC-learnable with 0-1 loss $\ell_{\operatorname{0-1}}$, training data $S$,  and algorithm $A$, then suppose that $\hat{h} = A(S)$ is a hypothesis estimated form $S$ satisfying
$L_{\mathcal{D}}(\hat{h}) \leq \epsilon$ for some parameter $\epsilon \in (0, 1)$. Then for any requirement $c$, hypothesis $\hat{h}_c$ obtained by modifying $\hat{h}$ with a CV satisfies
\begin{align*}
    L_{\mathcal{D}}(\hat{h}_c) \leq \min_{h_c \in \mathcal{H}_c} L_{\mathcal{D}}(h_c) + \epsilon \,.
\end{align*}
\end{theorem}
We give a proof in Appendix A. The proof bounds $L_{\mathcal{D}}(\hat{h}_c)$ using the fact that it is close to $L_\mathcal{D}(f_c)$, where $f_c$ is obtained by modifying $f : \mathcal{X} \to \mathcal{Y}$ to satisfy $L_\mathcal{D}(f) = 0$. The theorem suggests that if $\mathcal{H}$ is PAC-learnable, then inference time verification is sufficient to obtain a hypothesis with small generalization error in  $\mathcal{H}_c$.

Note that the generalization error might increase with a verifier, and the amount of the increase is always larger than $L_{\mathcal{D}}(f_c)$. Therefore, $L_{\mathcal{D}}(f_c)$ represents the discrepancy between data distribution $\mathcal{D}$ and requirement $c$, which is {\em consistent} with $f$ if $c(x, f(x)) = 1$ for all $x$. If $c$ is consistent with $f$, then $L_{\mathcal{D}}(f_c) = 0$, and we can certify that $L_\mathcal{D}(\hat{h}_c) \leq \epsilon$.

The above theorem shows that ITV works when $\mathcal{H}$ is PAC-learnable. However, this will not hold if $\mathcal{D}$ is not realizable with $\mathcal{H}$, i.e, $\mathcal{H}$ is not PAC-learnable.
\begin{theorem}
\label{thm:apac}
If $\mathcal{Y} = [K]$, the loss function is $\operatorname{0-1}$ loss $\ell_{\operatorname{0-1}}$ and hypothesis class $\mathcal{H}$ is not realizable with $\mathcal{D}$, and then
there exists training data $S$, algorithm $A$, requirements $c$, and $\epsilon \in (0, 1)$ such that $\hat{h} = A(S)$ satisfies $L_{\mathcal{D}}(\hat{h}) \leq \min_{h \in \mathcal{H}} L_{\mathcal{D}}(h) + \varepsilon$ but 
$L_{\mathcal{D}}(\hat{h}_c) > \min_{h_c \in \mathcal{H}_c}L_{\mathcal{D}}(h_c) + \epsilon$.
\end{theorem}
We give in Appendix B a proof that shows a counterexample even if $c$ is consistent with ground truth $f$.
The above theorems show that the realizability of $\mathcal{H}$ is the key factor that distinguishes among the cases where ITV works well. Moreover, unlike the realizable case, Theorem~\ref{thm:apac} holds even if requirement $c$ is consistent with distribution $\mathcal{D}$. Let $f_{\mathcal{D}} : \mathcal{X} \to \mathcal{Y}$ be defined as the Bayes optimal predictor:
\begin{equation*}
    f_{\mathcal{D}}(x) \defeq \argmax_{y \in \mathcal{Y}} \mathbb{P}[y \mid x]\,.
\end{equation*}
The Bayes optimal predictor is optimal, in the sense that for every other classifier $g: \mathcal{X} \to \mathcal{Y}$, $L_{\mathcal{D}}(f_{\mathcal{D}}) \leq L_{\mathcal{D}}(g)$.
Theorem~\ref{thm:apac} holds if $c$ is consistent with $f_{\mathcal{D}}$. These results show that existing methods
\cite{hokamp-liu-2017-lexically,anderson-etal-2017-guided} using constraints only in the inference time might
fail to select the best hypothesis. 

\textbf{Running time analysis:} Using a CV increases the time needed for inference. Suppose that a verifier is an oracle that can answer the query about the 
value of $c(x, y)$. To achieve a previously shown modification procedure (\ref{eq:modify}), we need at most $K$ queries.

\section{Learning Time Verification}
In Section~\ref{seq:ltv}, we show that if $\mathcal{H}$ is PAC-learnable with 0-1 loss, then modifying a hypothesis at the inference time is sufficient to obtain a hypothesis with the smallest generalization error while satisfying the requirements.
If $\mathcal{H}$ is not PAC-learnable, then the ITV scheme might fail to obtain a hypothesis with small generalization error. Here we show that the generalization error can be bounded when we use a CV in the learning phase. We call this setting {\em learning time verification} (LTV). 

Since the LTV scheme corresponds to a learning task where the hypothesis class is $\mathcal{H}_c$, 
we analyze the learnability of $\mathcal{H}_c$ using the standard tools for generalization analyses. This paper provides analyses based on 
Rademacher complexity since its a  widely used tools that can give tight bounds for both data-dependent and data-independent cases. Moreover, some previous work gives bounds of structured prediction tasks using
Rademacher complexity. In the literature, constraints are actively used in structured prediction tasks, including
language generation and  sequence labeling. Therefore, analyzing the generalization error is important when using a CV on structured prediction tasks.

In the following, we first show the upper bounds of generation error based on the Rademacher complexity of $\mathcal{H}_c$
in a multi-class classification task (\S 6.1, 6.2)  and a structured prediction setting (\S 6.3). 
Our main finding is that the upper bounds based on the Rademacher complexity of $\mathcal{H}_c$ are always less than or equal to those of $\mathcal{H}$. Therefore, adding a CV to a machine learning model will not degrade its learnability.

\subsection{Multi-class Classification}

We first give the Rademacher complexity-based error bounds on a multi-class classification task, i.e., $\mathcal{Y} = [K]$. In this section, we
show that a standard upper bound  \cite{10.5555/2371238-mohri} based on the Rademacher complexity of $\mathcal{H}$  can be used as an upper bound of $\mathcal{H}_c$ for any requirement $c$. 
In the next section, we show that a state-of-the-art error bound, based on local Rademacher complexity $\mathcal{H}$, can also be 
used as an upper bound of $\mathcal{H}_c$.  

Following previous works, let $h: (\mathcal{X} \times \mathcal{Y}) \to \mathbb{R}$ be a scoring function, and define hypothesis class $\mathcal{H}$ as a set of scoring functions. A scoring function defines a mapping from $\mathcal{X}$ to $\mathcal{Y}$:
\begin{align*}
    x \mapsto \argmax_{y \in \mathcal{Y}} h(x, y) \,.
\end{align*}
Let $\rho_{h}(x, y)$ be {\em the margin of function} of $h$:
\[
    \rho_h(x, y) \defeq h(x, y) - \max_{y^\prime \neq y}h(x, y^\prime) \,.
\]
Hypothesis $h$ misclassifies the labeled example $(x, y)$ if $\rho_{h}(x, y) \leq 0$. Thus, by using a margin function, the 0-1 loss can be represented as $\ell_{\operatorname{0-1}}(h, z) = \mathbf{1}_{\rho_h(x, y) \leq 0}$. Since $\operatorname{0-1}$ loss is hard to handle during learning, we use margin loss $\ell_{\rho}(h, (x, y)) = \Phi_\rho(\rho_h(x, y))$, where $\Phi_\rho(t)$ is defined as
\begin{equation*}
\Phi_\rho(t) = \min(1, \max(0, 1 - t/\rho)) \,.
\end{equation*}
 Function $f: \mathbb{R} \to \mathbb{R}$ is said to be {\em $\mu$-Lipschitz} if $|f(t) - f(t^\prime)| \leq \mu |t - t^\prime|$ for any $t, t^\prime \in \mathbb{R}$. $\Phi_{\rho}$ is an $1/\rho$-Lipschitz function. The {\em empirical margin loss} of hypothesis $h$ is defined as
\begin{align*}
    L_{S, \rho}(h) \defeq \frac{1}{m}\sum_{i=1}^{m} \Phi_{\rho}(\rho_{h}(x_i, y_i)) \,.
\end{align*}

Identical to the case of ITV, introducing a CV to a machine learning model corresponds to 
modifying its corresponding hypothesis class $\mathcal{H}$ to hypothesis class $\mathcal{H}_c$ that is consistent with requirement $c$. If $h$ is a score function, then we define 
consistent function $h_c$:
\begin{align}
\label{eq:modify-h-multiclass}
    h_c(x, y) = \left\{
    \begin{array}{cc}
         h(x, y)&  \mbox{ if } c(x, y) = 1 \\
         -M & \mbox{ if } c(x, y) = 0
    \end{array}
    \right. \,,
\end{align}
where $M$ is a positive constant satisfying $M >  |\max_{(x,y) \in Z}h(x, y)|$. As described in Section~\ref{sec:verifier},  we assume that there exists $y\in \mathcal{Y}$ that satisfies $c(x, y) = 1$ for all $x \in \mathcal{X}$. Therefore, we can guarantee that $\rho_{h_c}(x, y) < 0$ if $c(x, y) = 0$.

The following are the main results of the general multi-class learning problem, which is based on the margin bound
shown in Theorem 9.2 of Mohri et al. \cite{10.5555/2371238-mohri}. Our main finding is that the generalization error
of any hypothesis, $h_c$, is bounded by the Rademacher complexity of hypothesis set $\mathcal{H}$, which suggests
that if we have a tight bound for hypothesis class $\mathcal{H}$, then we can expect to find a good hypothesis
from $\mathcal{H}_c$ under any requirements $c$.
\begin{theorem}
Let $\mathcal{H} \subseteq \mathbb{R}^{\mathcal{X} \times \mathcal{Y}}$  be a hypothesis class with $\mathcal{Y} = [K]$, and let $c$ be a requirement. Fix $\rho > 0$. Then for any $\delta > 0$, with probability at least $1 - \delta$, the following bound holds for all $h_c \in \mathcal{H}_c$:
\begin{align*}
    L_{\mathcal{D}}(h_c) \leq  L_{S, \rho}(h_c) + \frac{4K}{\rho} R_{m}(\Pi_{1}(\mathcal{H})) + \sqrt{\frac{\log{\frac{1}{\delta}}}{2m}} \,,
\end{align*}
where $\Pi_1(\mathcal{H})$ is defined as
\begin{align*}
    \Pi_1(\mathcal{H}) \defeq \{x \mapsto h(x, y): y \in \mathcal{Y}, h \in \mathcal{H}\}\,.
\end{align*}
\end{theorem}

We give a proof in Appendix C. We obtain the results by showing that the upper bounds of the Rademacher complexity of $\mathcal{H}_c$ are bounded by some upper bounds of the Rademacher complexity of $\mathcal{H}$. All the proofs of the theorems in this section use similar techniques. Parameter $\rho$ sets the margin value. Following a previously shown technique \cite{10.5555/2371238-mohri}, we obtain a generalized bound that holds uniformly for all $\rho > 0$. The above theorem suggests that using a CV at 
a learning phase does not worsen the error bound for any requirement $c$. Intuitively, the theorem seems reasonable since requirements $c$ imposes a restriction on $\mathcal{H}$, and thus the complexity of $\mathcal{H}_c$ is not larger than $\mathcal{H}$. However, it is not so trivial since $\mathcal{H}_c \subseteq \mathcal{H}$ is not always true.

\textbf{Running time analysis:} We analyze the number of evaluations $c(x, y)$ required for learning with a CV. Let
$S_1$ be a sub-sequence of training example $S$ such that $c(x_i, y_i) = 1$, and let $S_0$ be a sub-sequence
such that $c(x_i, y_i) = 0$.  If we use a 0-1 loss function, then 
the empirical loss of hypothesis $h_c$ is
\begin{align*}
    L_S(h_c) = \frac{1}{m}\sum_{(x_i, y_i) \in S_1} \mathbf{1}_{h_c(x_i) \neq y_i}  + \frac{|S_0|}{m} \,,
\end{align*}
since every $h_c$  misclassifies the examples in $S_0$.  Therefore, we need at most $K|S_1| + |S_0|$ queries for the learning process. This is also true when we use a margin loss function. On the other hand, the problem of estimating the best hypothesis might be more difficult than the original problem depending on requirement $c$. 

\subsection{Tighter Bound Based on Local Rademacher Complexity}
The bound for $\mathcal{H}_c$ shown in the previous section is relatively simple, and tighter bounds 
of $\mathcal{H}$ based on the Rademacher complexity have been
developed in the literature. In this section, we show that the state-of-the-art error bound for $\mathcal{H}$ based on {\em the local Rademacher complexity} can be used as a bound for $\mathcal{H}_c$ for any requirements $c$.
\begin{definition}
Let $\mathcal{G}$ be a family of functions from $Z$ to $\mathbb{R}$,  and let $S$ be training data of size m. Then 
for any $r > 0$, {\em the empirical local Rademacher complexity} of $\mathcal{G}$ is defined as
\begin{align*}
    R_S(\mathcal{G}; r) = R_{S}\left(\{a g : a \in [0, 1],  g \in \mathcal{G}, \mathbb{E}[(ag)^2] \leq r \}\right)\,.
\end{align*}
\end{definition}
Li et al. \cite{NEURIPS2018_1141938b-multiclass} showed a tighter generalization bound for a multi-class classification problem using the local Rademacher complexity when the hypothesis class is a $\ell_p$ norm hypothesis space with kernel $\kappa$, defined as
\begin{align*}
    \mathcal{H}_{p, \kappa} \defeq \left\{h = (\langle \mathbf{w}_1, \phi(x) \rangle, \ldots,  \langle \mathbf{w}_K, \phi(x)\rangle) : \| \mathbf{w} \|_{2, p} \leq 1, 1 \leq p \leq 2  \right\} \,,
\end{align*}
where $h$ is represented as a vector valued function $(h_1, \ldots, h_K)$ with $h_j(x) = h(x, j), \forall j = 1, \ldots, K$, and $\kappa : \mathcal{X} \times \mathcal{X} \to \mathbb{R}$  is a Mercer kernel with associated feature map $\phi$, i.e., $\kappa(x, x^\prime) = \langle \phi(x), \phi(x^\prime)\rangle$. $\mathbf{w} = (\mathbf{w}_1, \ldots, \mathbf{w}_K)$, and $\| \mathbf{w}\| = \left[\sum_{i=1}^{K}\|\mathbf{w}\|_2^p\right]^\frac{1}{p}$ is the $\ell_{2, p}$-norm. For any $p \geq 1$, let
$q$ be the dual exponent of $p$ satisfying $1/p + 1/q = 1$. 
Let $\Phi: \mathbb{R} \to \mathbb{R}$ be a loss function satisfying the following: 1) $\mathbf{1}_{t\leq 0}(t) \leq \Phi(t)$ for all $t$; 2) $\Phi(t)$ is decreasing and has zero point $c_\Phi$; 3) $\Phi$ is $\zeta$-smooth, that is, $|\Phi^{\prime}(t) - \Phi^{\prime}(t^\prime)| \leq \zeta |t - t^\prime|$. 

Let $\mathcal{H}_{p, \kappa, c}$ be the hypothesis class obtained by modifying hypothesis $\mathcal{H}_{p, \kappa}$ to satisfy requirements $c$, and  $\mathcal{L}_c \defeq \{(x, y) \mapsto \Phi(\rho_{h_c}(x, y)) : h_c \in \mathcal{H}_{p, \kappa, c}\}$. The following theorem gives a bound of the local Rademacher complexity of $\mathcal{L}_c$.
\begin{theorem}
Let $\mathcal{H}_{p, \kappa, c}$ be the set of hypotheses obtained by modifying hypothesis $h \in \mathcal{H}_{p, \kappa}$ with requirement $c$. For any $\delta > 0$, with probability at least $1-\delta$, the following bound holds:
\begin{align*}
    R_m(\mathcal{L}_c; r) \leq \frac{C_{d, \vartheta} \xi(K) \sqrt{\zeta r}\log^\frac{3}{2}(m)}{\sqrt{m}} + \frac{4 \log \frac{1}{\delta}}{m} \,,
\end{align*}
where $\vartheta = \sup_{x \in \mathcal{X}} \kappa(x, x) < \infty$, $d = \sup_{t \in \mathbb{R}} \Phi(t) < \infty$, and $C_{d, \vartheta}$ is a constant.  $\xi(K)$ is
\begin{align*}
    \xi(K) = \left\{
    \begin{array}{cc}
         \sqrt{e}(4\log K)^{1+\frac{1}{2\log K}} &  \mbox{ if } q \geq 2 \log K, \\
         (2q)^{1 + \frac{1}{q}} K^\frac{1}{q}&  \mbox{ otherwise }.
    \end{array}\right.
\end{align*}
\end{theorem}
We give a proof in Appendix D. The bound equals that of $R_{m}(\mathcal{L}; r)$ (Theorem 1 of \cite{NEURIPS2018_1141938b-multiclass}) for any requirement $c$ and any hypothesis $\mathcal{H}_{\rho, \kappa}$. Therefore, the generalization error bounds based on  Theorem 1 of Li et al. \cite{NEURIPS2018_1141938b-multiclass}  holds for any requirement $c$.

\subsection{Analyses of Structured Prediction}

Structured prediction is a kind of multi-class classification task, where label set $\mathcal{Y}$ might be a set
of sequences, images, graphs, trees, or other objects admitting some possibly overlapping structure. 
As mentioned in Section~\ref{seq:introduction}, previous works try to impose constraints on the output of structured
prediction tasks. Thus it is also useful to derive error bounds for structured prediction tasks when we use a CV. In the following, we show that the Rademacher complexity-based generalization error bounds derived in a seminal work of Cortes et al. ~\cite{NIPS2016_535ab766-structure} also hold if we use a CV. 
Although tighter bounds are given in a more recent work~\cite{ijcai2021-0391-structure,NEURIPS2021_structure}, we give bounds based on Cortes et al.~\cite{NIPS2016_535ab766-structure} due to their simplicity.

We give some definitions for the structured prediction task. Following previous work, we assume that $\mathcal{Y}$ is
decomposable along with substructures: $\mathcal{Y} = \mathcal{Y}_1 \times \cdots \times \mathcal{Y}_l$. Here $\mathcal{Y}_k$
is a set of possible labels that can be assigned to the $k$-th substructure. We denote by $\mathsf{L}: \mathcal{Y} \times \mathcal{Y} \to \mathbb{R}_+$ a loss function that measures the dissimilarity of two elements of output space $\mathcal{Y}$.  $\mathsf{L}$  is {\em definite}, that is, $\mathsf{L}(y, y^\prime) = 0$ iff $y = y^\prime$. A typical definite loss function 
for a structured prediction task is the Hamming loss defined by $\mathsf{L}(y, y^\prime) = \frac{1}{l} \sum_{k=1}^{l} \mathbf{1}_{y_k \neq y_k^\prime}$  for all $y = (y_1, \ldots, y_l)$ and $y^\prime = (y_1^\prime, \ldots, y_l^\prime)$, with $y_k, y_k^\prime \in \mathcal{Y}_k$.
Other typical examples of loss functions can be seen in Cortes et al. \cite{NIPS2016_535ab766-structure}. Using loss function $\mathsf{L}$, the generalization and empirical error of $h$ are defined:
\begin{align*}
    L_{\mathcal{D}}(h) = \expected_{(x, y) \sim \mathcal{D}}[\mathsf{L}(\mathsf{h}(x), y)], ~~~~~    L_{S}(h) = \frac{1}{m} \sum_{i=1}^{m} \mathsf{L}(\mathsf{h}(x), y) \,.
\end{align*}

As with the multi-class classification task, hypothesis class $\mathcal{H}$ can be represented as a set of scoring function $h : \mathcal{X} \times \mathcal{Y} \to \mathbb{R}$. We use $\mathsf{h}(x)$ to represent the predictor defined by $h \in \mathcal{H}$: $\mathsf{h}(x) \defeq \argmax_{y\in \mathcal{Y}} h(x, y)$ for all $x \in X$. Following the previous work, we assume that each scoring function can be decomposed as a sum, and such decomposition follows
a {\em factor graph}. Factor graph $G$ is a tuple $G = (V, F, E)$, where $V$ is a set of variable nodes, $F$ is a set of factor nodes,
and $E$ is a set of undirected edges between a variable node and a factor node. Every node in $V$ corresponds to a substructure index, where 
$V = \{1, \ldots, l\}$. 

\begin{figure}[t]
\centering
\includegraphics[width=0.65\hsize,keepaspectratio]{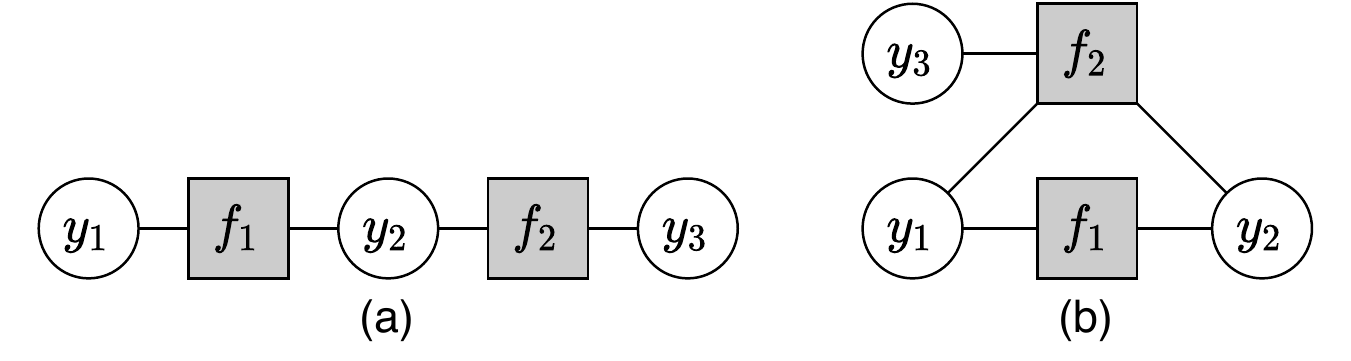}
\caption{Example of factor graphs: (a) represents decomposition $h(x, y) = h_{f_1}(x, y_1, y_2) + h_{f_2}(x, y_2, y_3)$ and (b) represents decomposition $h(x, y) = h_{f_1}(x, y_1, y_2) + h_{f_2}(x, y_1, y_2, y_3)$.}
\label{fig:factor}
\end{figure}

For any factor node $f$, we denote by $\mathcal{N}(f) \subseteq V$ a set of variable nodes connected to $f$ and define $\mathcal{Y}_f$ as substructure set cross-product $\mathcal{Y}_f = \prod_{k\in \mathcal{N}(f)} \mathcal{Y}_k$. Then $h$ admits the following decomposition 
as a sum of functions $h_f$, each taking as an argument a pair of $(x, y_f) \in \mathcal{X} \times \mathcal{Y}_f$:
\begin{align}
\label{eq:factor}
    h(x, y) = \sum_{f \in F} h_f(x, y_f)\,.
\end{align}
Figure~\ref{fig:factor} shows examples of decompositions based on factor graphs. We conventionally assume that the structure of the factor graphs may 
change depending on a particular example $(x_i, y_i)$: $G(x_i, y_i) = G_i = ([l_i], F_i, E_i)$. A special case of this setting
is when size $l_i$ of each example is allowed to vary. In such a case, the number of possible labels $\mathcal{Y}$ is potentially infinite. 

Following multi-class classification, our CV maps hypothesis $h$ to $h_c$ to satisfy the 
requirements. The definition of $h_c$ follows Eq. (\ref{eq:modify-h-multiclass}). This definition does not require $h_c$ to have a factored representation.

For analyzing the complexity, Cortes et al. \cite{NIPS2016_535ab766-structure} introduced {\em empirical factor graph Rademacher complexity} $R_{S}^{G}(\mathcal{H})$ of
hypothesis class $\mathcal{H}$ for  $S = (x_1, \ldots, x_m)$ and factor graph $G$:
\[
    R_{S}^{G}(\mathcal{H}) = \frac{1}{m}\expected_{\boldsymbol \epsilon}\left[ 
    \sup_{h \in \mathcal{H}} \sum_{i=1}^m \sum_{f \in F_i} \sum_{y \in \mathcal{Y}_f} \sqrt{|F_i|} \epsilon_{i, f, y} h_f(x_i, y)
    \right] \,,
\]
where ${\boldsymbol \epsilon} = (\epsilon_{i, f, y})_{i \in [m], f\in F_i, y \in \mathcal{Y}_f}$ and every $\epsilon_{i, f, y}$ is i.i.d. a Rademacher random variable. Factor graph Rademacher complexity of $\mathcal{H}$ for factor graph $G$ is defined as expectation 
\[R_{m}^{G}(\mathcal{H}) \defeq \expected_{S \sim \mathcal{D}^m}[R_S^{G}(\mathcal{H})] \,.
\]

By using the factor graph Rademacher complexity, Cortes et al. \cite{NIPS2016_535ab766-structure} 
gives bounds for a structured prediction task with the 
following additive and multiplicative empirical losses:
\begin{align*}
   L_{S, \rho}^{\mathrm{add}}(h) \defeq &\frac{1}{m} \sum_{i=1}^{m}
    \left[\Phi^\ast \left(\max_{y^\prime \neq y_i} \mathsf{L}(y^\prime, y_i) -  
    \frac{1}{\rho}\left(h(x_i, y_i) - h(x_i, y^\prime)\right)  \right)
    \right]\displaybreak[1]\\
    L_{S, \rho}^{\mathrm{mult}}(h) \defeq & \frac{1}{m} \sum_{i=1}^{m}
    \left[\Phi^\ast \left(\max_{y^\prime \neq y_i} \mathsf{L}(y^\prime, y_i) \left(1-
    \frac{1}{\rho}\left(h(x_i, y_i) - h(x_i, y^\prime)\right)\right)\right)
    \right]     \, ,
\end{align*}
where   $\Phi^\ast(t) = \min(B, \max(0, t))$ for all $t$, with $B = \max_{y, y^\prime} \mathsf{L}(y, y^\prime)$. As shown in  \cite{NIPS2016_535ab766-structure}, these loss functions cover typical surrogate loss functions used in structured prediction tasks.
We show the following bound for structured predictions.
\begin{theorem}
Fix $\rho > 0$. For any $\delta > 0$ and requirement $c$, with probability at least $1 - \delta$ over the draw of sample $S$ of size $m$ from distribution $\mathcal{D}$, the following holds for all $h_c \in \mathcal{H}_c$:
\begin{align*}
    L_\mathcal{D}(h_c) &\leq L_{\mathcal{D}, \rho}^{\mathrm{add}}(h_c) \leq L_{S, \rho}^{\mathrm{add}}(h_c) + \frac{4\sqrt{2}}{\rho} R_{m}^{G}(\mathcal{H}) + B\sqrt{\frac{\log \frac{1}{\delta}}{2m}} \\
    L_\mathcal{D}(h_c) &\leq L_{\mathcal{D}, \rho}^{\mathrm{mult}}(h_c) \leq L_{S, \rho}^{\mathrm{mult}}(h_c) + \frac{4\sqrt{2}B}{\rho} R_{m}^{G}(\mathcal{H}) + B\sqrt{\frac{\log \frac{1}{\delta}}{2m}} \, .
\end{align*}
\end{theorem}
We give a proof in Appendix E. $\rho$ is a parameter that determines the margin. Similar to the case of multi-class classification, we can derive a bound that holds for any $\rho > 0$ following a previous derivation \cite{NIPS2016_535ab766-structure}.
The above result indicates that the bound will not change if we use
a CV for any requirement $c$. This is interesting since the above result holds even if we do not have a factored representation of $h_c(x, y)$, similar to Eq. (\ref{eq:factor}), although the derived bound depends on the factor graph Rademacher complexity, which depends on the factored representation of $h(x, y)$. 

We analyzed the overhead of the running time for evaluating loss function $L_{S, \rho}^{\mathrm{add}}(h_c)$ and $L_{S, \rho}^{\mathrm{mult}}(h_c)$ for hypothesis $h_c$. Different from the multi-class classification case, both the number of queries and the overhead of the running time for the loss evaluation when we use a CV depend on the model and the type of requirements for structured predictions. This result is consistent with the literature, which reports that for structured prediction tasks, original
tractable optimization problems can be intractable if we put additional constraints~\cite{roth2005integer}. 

\section{Conclusion}

This paper gives a generalization analysis when there are requirements that the input-output pairs of a machine learning model must satisfy. We introduce a concurrent verifier, a simple module that enables us to 
guarantee that the input-output pairs of a machine learning model satisfy the requirements. 
We show a situation where we can obtain a hypothesis with small error when we use a verifier only in the
inference phase. Interestingly, if $\mathcal{H}$ is not PAC-learnable, we might fail to obtain a guaranteed hypothesis
even if the requirements are consistent with distribution $\mathcal{D}$. 
We also give the generalization bounds based on Rademacher complexity when we use a
verifier in a learning phase and find that the obtained bounds are less than or equal to the existing 
ones, independent of the machine learning model and the type of requirements. 

\section*{Acknowledgements}
The authors thank the anonymous reviewers for their valuable feedback, corrections, and suggestions.
This work was supported by JST PRESTO (Grant Number JPMJPR20C7, Japan) and JSPS KAKENHI (Grant Number JP20H05963, Japan).
{
\small
\bibliographystyle{plain}
\bibliography{neurips_2022}

\begin{thebibliography}{10}

\bibitem{Alquier2021UserfriendlyIT-pacbayes}
Pierre Alquier.
\newblock User-friendly introduction to pac-bayes bounds.
\newblock {\em ArXiv}, abs/2110.11216, 2021.

\bibitem{anderson-etal-2017-guided}
Peter Anderson, Basura Fernando, Mark Johnson, and Stephen Gould.
\newblock Guided open vocabulary image captioning with constrained beam search.
\newblock In {\em Proceedings of the 2017 Conference on Empirical Methods in
  Natural Language Processing}, pages 936--945, Copenhagen, Denmark, September
  2017. Association for Computational Linguistics.

\bibitem{10.5555/944919.944944-rademacher1}
Peter~L. Bartlett and Shahar Mendelson.
\newblock Rademacher and gaussian complexities: Risk bounds and structural
  results.
\newblock {\em J. Mach. Learn. Res.}, 3:463^^e2^^80^^93482, 2003.

\bibitem{NEURIPS2018_be53d253-linear-network-verificaton}
Rudy~R Bunel, Ilker Turkaslan, Philip Torr, Pushmeet Kohli, and Pawan~K
  Mudigonda.
\newblock A unified view of piecewise linear neural network verification.
\newblock In S.~Bengio, H.~Wallach, H.~Larochelle, K.~Grauman, N.~Cesa-Bianchi,
  and R.~Garnett, editors, {\em Advances in Neural Information Processing
  Systems}, volume~31. Curran Associates, Inc., 2018.

\bibitem{10.1007/s10994-012-5296-5-structured}
Ming-Wei Chang, Lev Ratinov, and Dan Roth.
\newblock Structured learning with constrained conditional models.
\newblock {\em Mach. Learn.}, 88(3):399^^e2^^80^^93431, sep 2012.

\bibitem{NIPS2016_535ab766-structure}
Corinna Cortes, Vitaly Kuznetsov, Mehryar Mohri, and Scott Yang.
\newblock Structured prediction theory based on factor graph complexity.
\newblock In {\em Advances in Neural Information Processing Systems},
  volume~29. Curran Associates, Inc., 2016.

\bibitem{10.5555/2789272.2912074-daniely}
Amit Daniely, Sivan Sabato, Shai Ben-David, and Shai Shalev-Shwartz.
\newblock Multiclass learnability and the erm principle.
\newblock {\em J. Mach. Learn. Res.}, 16(1):2377^^e2^^80^^932404, jan 2015.

\bibitem{10.5555/1625275.1625673-problog}
Luc De~Raedt, Angelika Kimmig, and Hannu Toivonen.
\newblock Problog: A probabilistic prolog and its application in link
  discovery.
\newblock In {\em Proceedings of the 20th International Joint Conference on
  Artifical Intelligence}, page 2468^^e2^^80^^932473, San Francisco, CA, USA,
  2007. Morgan Kaufmann Publishers Inc.

\bibitem{hokamp-liu-2017-lexically}
Chris Hokamp and Qun Liu.
\newblock Lexically constrained decoding for sequence generation using grid
  beam search.
\newblock In {\em Proceedings of the 55th Annual Meeting of the Association for
  Computational Linguistics (Volume 1: Long Papers)}, pages 1535--1546,
  Vancouver, Canada, July 2017. Association for Computational Linguistics.

\bibitem{10.1007/978-3-319-63387-9_5-relplex}
Guy Katz, Clark Barrett, David~L. Dill, Kyle Julian, and Mykel~J. Kochenderfer.
\newblock Reluplex: An efficient smt solver for verifying deep neural networks.
\newblock In Rupak Majumdar and Viktor Kun{\v{c}}ak, editors, {\em Computer
  Aided Verification}, pages 97--117, Cham, 2017. Springer International
  Publishing.

\bibitem{KR148005-psdd}
Doga Kisa, Guy~Van den Broeck, Arthur Choi, and Adnan Darwiche.
\newblock Probabilistic sentential decision diagrams.
\newblock In {\em Knowledge Representation and Reasoning Conference}, 2014.

\bibitem{10.2307/2700001-rademacher2}
V.~Koltchinskii and D.~Panchenko.
\newblock Empirical margin distributions and bounding the generalization error
  of combined classifiers.
\newblock {\em The Annals of Statistics}, 30(1):1--50, 2002.

\bibitem{ledoux1991probability-talagrand}
Michel Ledoux and Michel Talagrand.
\newblock {\em Probability in Banach Spaces: isoperimetry and processes},
  volume~23.
\newblock Springer Science \& Business Media, 1991.

\bibitem{NIPS2015_3a029f04-multi-svm}
Yunwen Lei, Urun Dogan, Alexander Binder, and Marius Kloft.
\newblock Multi-class svms: From tighter data-dependent generalization bounds
  to novel algorithms.
\newblock In {\em Advances in Neural Information Processing Systems},
  volume~28. Curran Associates, Inc., 2015.

\bibitem{NEURIPS2018_1141938b-multiclass}
Jian Li, Yong Liu, Rong Yin, Hua Zhang, Lizhong Ding, and Weiping Wang.
\newblock Multi-class learning: From theory to algorithm.
\newblock In S.~Bengio, H.~Wallach, H.~Larochelle, K.~Grauman, N.~Cesa-Bianchi,
  and R.~Garnett, editors, {\em Advances in Neural Information Processing
  Systems}, volume~31. Curran Associates, Inc., 2018.

\bibitem{NEURIPS2021_structure}
Shaojie Li and Yong Liu.
\newblock Towards sharper generalization bounds for structured prediction.
\newblock In {\em Advances in Neural Information Processing Systems},
  volume~34. Curran Associates, Inc., 2021.

\bibitem{10.1145/279943.279989-pacbayes-origin}
David~A. McAllester.
\newblock Some pac-bayesian theorems.
\newblock In {\em Proceedings of the Eleventh Annual Conference on
  Computational Learning Theory}, page 230^^e2^^80^^93234, New York, NY, USA,
  1998. Association for Computing Machinery.

\bibitem{10.5555/2371238-mohri}
Mehryar Mohri, Afshin Rostamizadeh, and Ameet Talwalkar.
\newblock {\em Foundations of Machine Learning}.
\newblock The MIT Press, 2012.

\bibitem{ijcai2021-0391-structure}
Waleed Mustafa, Yunwen Lei, Antoine Ledent, and Marius Kloft.
\newblock Fine-grained generalization analysis of structured output prediction.
\newblock In {\em Proceedings of the Thirtieth International Joint Conference
  on Artificial Intelligence}, pages 2841--2847. International Joint
  Conferences on Artificial Intelligence Organization, 8 2021.
\newblock Main Track.

\bibitem{Narodytska2020In-bnn-sat}
Nina Narodytska, Hongce Zhang, Aarti Gupta, and Toby Walsh.
\newblock In search for a sat-friendly binarized neural network architecture.
\newblock In {\em International Conference on Learning Representations}, 2020.

\bibitem{10.1023/A:1022605311895-natarajan}
B.~K. Natarajan.
\newblock On learning sets and functions.
\newblock {\em Mach. Learn.}, 4(1):67^^e2^^80^^9397, 1989.

\bibitem{richardson2006markov}
Matthew Richardson and Pedro Domingos.
\newblock Markov logic networks.
\newblock {\em Mach. Learn.}, 62(1-2):107--136, 2006.

\bibitem{roth2005integer}
Dan Roth and Wen-tau Yih.
\newblock Integer linear programming inference for conditional random fields.
\newblock In {\em Proceedings of the 22nd international conference on Machine
  learning}, pages 736--743, 2005.

\bibitem{10.5555/2621980-understanding}
Shai Shalev-Shwartz and Shai Ben-David.
\newblock {\em Understanding Machine Learning: From Theory to Algorithms}.
\newblock Cambridge University Press, USA, 2014.

\bibitem{10.5555/1756006.1953019-stability}
Shai Shalev-Shwartz, Ohad Shamir, Nathan Srebro, and Karthik Sridharan.
\newblock Learnability, stability and uniform convergence.
\newblock {\em J. Mach. Learn. Res.}, 11:2635^^e2^^80^^932670, dec 2010.

\bibitem{NEURIPS2018_f2f44698-verify}
Gagandeep Singh, Timon Gehr, Matthew Mirman, Markus P\"{u}schel, and Martin
  Vechev.
\newblock Fast and effective robustness certification.
\newblock In S.~Bengio, H.~Wallach, H.~Larochelle, K.~Grauman, N.~Cesa-Bianchi,
  and R.~Garnett, editors, {\em Advances in Neural Information Processing
  Systems}, volume~31. Curran Associates, Inc., 2018.

\bibitem{42503intriguring-adversarial}
Christian Szegedy, Wojciech Zaremba, Ilya Sutskever, Joan Bruna, Dumitru Erhan,
  Ian Goodfellow, and Rob Fergus.
\newblock Intriguing properties of neural networks.
\newblock In {\em International Conference on Learning Representations}, 2014.

\bibitem{tjeng2018evaluating-mipverify}
Vincent Tjeng, Kai~Y. Xiao, and Russ Tedrake.
\newblock Evaluating robustness of neural networks with mixed integer
  programming.
\newblock In {\em International Conference on Learning Representations}, 2019.

\bibitem{doi:10.1137/1116025vc}
V.~N. Vapnik and A.~Ya. Chervonenkis.
\newblock On the uniform convergence of relative frequencies of events to their
  probabilities.
\newblock {\em Theory of Probability \& Its Applications}, 16(2):264--280,
  1971.

\end{thebibliography}
}

\appendix

\section{Proof of Theorem 5.1}
\begin{proof}
Since $\mathcal{D}$ is realizable, $f : \mathcal{X} \to \mathcal{Y}$ exists such that $L_{\mathcal{D}}(f) = 0$. 
We show that $L_{\mathcal{D}}(f_c) \leq L_{\mathcal{D}}(\hat{h}_c) \leq L_{\mathcal{D}}(f_c) + \epsilon$. 

We first prove $L_{\mathcal{D}}(f_c) \leq L_{\mathcal{D}}(\hat{h}_c)$. From the definition of $f_c$, if $f(x) \neq f_c(x)$, then $c(x, f(x)) = 0$, and  thus $h_c(x) \neq f(x)$ for all $h_c \in \mathcal{H}_c$. Therefore,  
 for all $h_c \in \mathcal{H}_c$,
\begin{equation*}
L_{\mathcal{D}}(f_c) = \mathbb{P}[f_c(x) \neq f(x)] \leq \mathop{\mathbb{P}}[h_c(x) \neq f(x)] = L_{\mathcal{D}}(h_c).
\end{equation*}
Next we prove  $ L_{\mathcal{D}}(\hat{h}_c) \leq L_{\mathcal{D}}(f_c) + \epsilon$. For all $h \in \mathcal{H}$, we show that  $L_{\mathcal{D}}(h_c) - L_\mathcal{D}(h)$ is bounded:
\begin{align*}
L_{\mathcal{D}}(h_c) - L_{\mathcal{D}}(h) \leq
\mathbb{P}[h_c(x) \neq h(x) \mbox{ and } h(x) = f(x) ] \nonumber
\leq \mathbb{P}[f_c(x) \neq f(x)] = L_\mathcal{D}(f_c) \, ,
\end{align*}
where the first inequality uses the fact that the error increases if we modify the output at $x$ and satisfy $f(x) = h(x)$. 
Since $c(x, f(x)) = 0$ for such $x$, the probability is less than error $L_{\mathcal{D}}(f_c)$.
Thus, $L_\mathcal{D}(\hat{h}_c) \leq L_{\mathcal{D}}(\hat{h}) + L_{\mathcal{D}}(f_c) \leq L_{\mathcal{D}}(f_c) + \epsilon$. 
\end{proof}

\section{Proof of Theorem 5.2}
\begin{proof}
Suppose that $\mathcal{H} = \{h_0, h_1\}$, and $f \not \in \mathcal{H}$ exists, satisfying $L_\mathcal{D}(f) = 0$. 
Suppose partition $\mathcal{X}_0, \mathcal{X}_1$ of $\mathcal{X}$ exists such that $h_{0}(x) \neq f(x)$ iff $x \in \mathcal{X}_0$ and $h_{1}(x) \neq f(x)$  iff $x \in \mathcal{X}_1$. Suppose that $|L_{\mathcal{D}}(h_0)  - L_{\mathcal{D}}(h_1)| \leq \epsilon$ for some $\epsilon \in (0, 1)$. If we design $c$ such that $c(x, y) = 0$  iff 
$x \in \mathcal{X}_0$ and $y = h_0(x)$, otherwise $c(x, y) = 1$. Then the generalization error of modified hypothesis $h_{c1}$ becomes $L_{\mathcal{D}}(h_{c0}) = 0$, and $L_{\mathcal{D}}(h_{c1}) = L_{\mathcal{D}}(h_{1})$. Thus
if $L_{\mathcal{D}}(h_1) > \epsilon$, and then difference  $|L_{\mathcal{D}}(h_{c0})  - L_{\mathcal{D}}(h_{c1})|$ becomes larger than $\epsilon$.
\end{proof}
Note that the above proof holds for $c$, which is consistent with $f$; that is, $c(x, f(x)) = 1$ for all $x \in \mathcal{X}$.

\section{Proof of Theorem 6.1}

We first introduce Talagrand's lemma with which we prove the main theorem.
\begin{lemma}[Talagrand's lemma, \cite{ledoux1991probability-talagrand,10.5555/2371238-mohri} ]
\label{lemma:talagrand}
Let $\Phi$ be the $\mu$-Lipschitz function from $\mathbb{R}$ to $\mathbb{R}$, and let $\sigma_1, \ldots, \sigma_m$ be Rademacher random variables. Then for any hypothesis set $\mathcal{H}$ of real-valued functions, the following inequality holds:
\begin{align*}
    R_{S}(\Phi \circ \mathcal{H})  \leq \mu R_{S}( \mathcal{H}) \,.
\end{align*}
\end{lemma}

We also use the following lemma.
\begin{lemma}[Lemma 9.1 of Mohri et al. \cite{10.5555/2371238-mohri}]
\label{lemma:maxmany}
Let $\mathcal{F}_1, \ldots, \mathcal{F}_l$ be $l$ hypothesis sets in $\mathbb{R}^{\mathcal{X}}$, $l \geq 1$, and let
$\mathcal{G} = \{\max\{h_1, \ldots, h_l\} : h_i \in \mathcal{F}_i, i \in [l] \}$. Then for any training data $S$ of size $m$, 
the empirical Rademacher complexity of $\mathcal{G}$ can be upper bounded:
\begin{align*}
    R_S(\mathcal{G}) \leq \sum_{j=1}^{l}R_S(\mathcal{F}_j) \,.
\end{align*}
We use $\max\{h_1, \ldots, h_l\}$ to represent the mapping from $\mathcal{X}$ to $\mathbb{R}$ that maps $x \in \mathcal{X}$ to $\max\{h_1(x), \ldots, h_l(x)\}$.
The above inequality holds if we use $\min$ instead of $\max$ in the definition of $\mathcal{G}$. 
\end{lemma}
The following lemma shows the relationship between the Rademacher complexities of 
$\mathcal{H}$ and $\mathcal{H}_c$.
\begin{lemma}
\label{lemma:max}
Let $\mathcal{H}$ be a hypothesis set in $\mathbb{R}^{\mathcal{X} \times \mathcal{Y}}$, and let $c : \mathcal{X} \times \mathcal{Y} \to \{0, 1\}$  be the requirements. Then for any training data $S = ((x_1, y_1), \ldots, (x_m, y_m))$ of size $m$, the following inequality holds:
\begin{align*}
R_S(\mathcal{H}_c) &= \expected_{\boldsymbol \sigma}\left[
\sup_{h_c \in \mathcal{H}_c}\sum_{i=1}^{m}\sigma_i h_c(x_i, y_i)
\right] \\
&\leq  \expected_{\boldsymbol \sigma}\left[
\sup_{h \in \mathcal{H}}\sum_{i=1}^{m}\sigma_i h(x_i, y_i)
\right] = R_S(\mathcal{H}) \,.
\end{align*}
\end{lemma}
\begin{proof}
Let $b_0(x, y): \mathcal{X} \times \mathcal{Y} \to \mathbb{R}$ be the following function:
\begin{align*}
    b_0(x, y) = \left\{ 
    \begin{array}{cc}
        M & \mbox { if } c(x, y) = 1\\
       -M & \mbox { if } c(x, y) = 0
    \end{array}
    \right. \,.
\end{align*}
Let $\mathcal{B}_0 = \{b_0\}$. Then  $\mathcal{H}_c$ can be represented as $\mathcal{H}_c = \{\min\{h, b_0\}: h\in \mathcal{H}, b_0 \in \mathcal{B}_0 \}$. From Lemma~\ref{lemma:maxmany},
\begin{align*}
    R_{S}(\mathcal{H}_c) \leq R_S(\mathcal{H}) + R_S(\mathcal{B}_0) \,.
\end{align*}
Since the empirical Rademacher complexity of singleton hypothesis class $\mathcal{B}_0$ is zero for any $S$, 
$R_S(\mathcal{H}_c) \leq R_S(\mathcal{H})$ holds.
\end{proof}

\begin{proof}[Proof of Theorem 6.1]
Let us define two sets of mappings, $\mathcal{H}_{c0}$ and $\mathcal{H}_{c1}$:
\begin{align*}
\mathcal{H}_{c0} &\defeq  \{(x, y) \mapsto \rho_{\theta, h_c}(x, y): h_c \in \mathcal{H}_c\},
\end{align*}
where we define $\rho_{\theta, h}(x, y)$:
\begin{align*}.
    \rho_{\theta, h}(x, y) \defeq \min_{y^\prime \in \mathcal{Y}}(h(x, y) - h(x, y^\prime) + \theta \mathbf{1}_{y^\prime = y}) \,,
\end{align*}
where $\theta > 0$ is an arbitrary constant. $\rho_{\theta, h}$ satisfies $\mathbb{E}[\mathbf{1}_{\rho_h(x, y) \leq 0}] \leq  \mathbb{E}[\mathbf{1}_{\rho_{\theta, h}(x, y) \leq 0} ]$ since  $\rho_{\theta, h}(x, y) \leq \rho_{h}(x, y)$ holds for
all $(x, y) \in \mathcal{X} \times \mathcal{Y}$.
Following the proof of Theorem 9.2 of Mohri et al. \cite{10.5555/2371238-mohri} with a probability of at least $1 - \delta$, for all $h_c \in \mathcal{H}_c$:
\begin{align*}
    L_{\mathcal{D}}(h_c) \leq L_{S}(h_c) + \frac{2}{\rho} R_{m}(\mathcal{H}_{c0}) + \sqrt{\frac{\log{\frac{1}{\delta}}}{2m}}\,.
\end{align*}
Thus, to complete the proof it suffices to show $R_{m}(\mathcal{H}_{c0}) \leq 2K R_{m}(\Pi_1(\mathcal{H}))$. We can upper bound $R_m(\mathcal{H}_{c0})$:
\begin{align*}
R_{m}(\mathcal{H}_{c0}) \leq 
\frac{1}{m} \mathop{\mathbb{E}}_{S, {\boldsymbol \sigma}}
\left[ 
\sup_{h_{c} \in \mathcal{H}_{c}}\sum_{i=1}^{m} \sigma_i h_c(x_i, y_i) 
\right] + 
\frac{1}{m} \mathop{\mathbb{E}}_{S, {\boldsymbol \sigma}}
\left[ 
\sup_{h_{c} \in \mathcal{H}_{c}}\sum_{i=1}^{m} \sigma_i \max_{y\in \mathcal{Y}}(h_c(x_i, y) - 2\rho \mathbf{1}_{y = y_i})
\right] \, .
\end{align*}
By applying lemma~\ref{lemma:max},  the first term is bounded by $R_{m}(\mathcal{H})$. Then we follow a previous
proof of Mohri et al. \cite{10.5555/2371238-mohri} to see that $R_{m}(\mathcal{H}) \leq K R_{m}(\Pi_{1}(\mathcal{H}))$. 
We bound the second term:
\begin{align*}
&\frac{1}{m}\mathop{\mathbb{E}}_{S, {\boldsymbol \sigma}} \left[ 
\sup_{h_c \in \mathcal{H}_c} \sum_{i=1}^{m} \sigma_i \max_{y \in \mathcal{Y}}(h_c(x_i, y) - 2\rho \mathbf{1}_{y=y_i})
\right] \\
 \leq& \sum_{y\in \mathcal{Y}}\frac{1}{m}\mathop{\mathbb{E}}_{S, {\boldsymbol \sigma}} 
\left[
\sup_{h_c \in \mathcal{H}_c} \sum_{i=1}^{m} \sigma_i (h_c(x_i, y) - 2\rho \mathbf{1}_{y=y_i})
\right] \\
=& \sum_{y\in \mathcal{Y}}\frac{1}{m}\mathop{\mathbb{E}}_{S, {\boldsymbol \sigma}} 
\left[
\sup_{h_c \in \mathcal{H}_c} \sum_{i=1}^{m} \sigma_i h_c(x_i, y) 
\right] \\
\leq& \sum_{y\in \mathcal{Y}}\frac{1}{m}\mathop{\mathbb{E}}_{S, {\boldsymbol \sigma}} 
\left[
\sup_{h_c \in \mathcal{H}_c} \sum_{i=1}^{m} \sigma_i h(x_i, y) 
\right] \leq K R_{m}(\Pi_1(\mathcal{H})) \, .
\end{align*}
We use lemma~\ref{lemma:maxmany} to derive the second line and lemma~\ref{lemma:max} to derive the
fourth line.
\end{proof}

\section{Proof of Theorem 6.3}

Our proof is made by changing a part of a previous proof of theorem 1 of Li et al. \cite{NEURIPS2018_1141938b-multiclass}. 
We first introduce empirical Gaussian complexity and a lemma.

\begin{definition}{(Empirical Gaussian complexity)}
Let $\mathcal{G}$ be a family of functions mapping from $Z$ to $\mathbb{R}$, and let
$S = (z_1, \ldots, z_m) \in Z^m$ be training data of size $m$. Then {\em the empirical 
Gaussian complexity} of $\mathcal{G}$ with respect to  $S$ is defined:
\[
\mathfrak{G}_{S}(\mathcal{G}) =  \frac{1}{m} \expected_{\boldsymbol g}
\left[
\sup_{f \in \mathcal{G}}  \sum_{i=1}^m g_i f(z_i)
\right] \,,
\] 
where $g_1, \ldots, g_m$ are independent $N(0, 1)$ random variables.
\end{definition}

We need the following lemma, which is based on lemma 4 of Lei et al. \cite{NIPS2015_3a029f04-multi-svm}.
\begin{lemma}
\label{lemma:gaussian}
Let $\mathcal{H}$ be a hypothesis class of mappings $\mathcal{X} \times \mathcal{Y} \to \mathbb{R}$, where $\mathcal{Y} = [K]$. $h \in \mathcal{H}$ is represented as 
vector $h = (h_{1}, \ldots, h_{K})$.  Let $c$ be a requirement, 
and let $g_1, \ldots, g_{mK}$ be  $N(0, 1)$ random variables. Then for any training data $S = (x_1, \ldots, x_m)$ of size $m$, we have:
\begin{align}
\label{eq:gaussian}
\mathfrak{G}_S(\{\max\{h_{c1}, \ldots, h_{cK}\} : h_c \in \mathcal{H}_c\})
\leq \frac{1}{m} \expected_{\boldsymbol g}
\left[
\sup_{h \in \mathcal{H}}
\sum_{i=1}^m \sum_{j=1}^{K} g_{(j-1)m+i} h_j(x_i)
\right] \,.
\end{align}
\end{lemma}
\begin{proof}
We make a proof by modifying the proof for lemma 4 of Lei et al. \cite{NIPS2015_3a029f04-multi-svm}. Define two 
Gaussian processes indexed by $\mathcal{H}_c$ and $\mathcal{H}$:
\begin{align*}
\mathfrak{X}_{h_{c}} &\defeq  \sum_{i=1}^{m}\left[
g_i \max\{h_{c1}(x_i), \ldots, h_{cK}(x_i)\}
\right] \,,\\
\mathfrak{Y}_{h} &\defeq  \sum_{i=1}^{m}\sum_{j=1}^K 
g_{(j-1)m+i} h_j(x_i) \,.
\end{align*}
For any $h = (h_1, \ldots, h_K)$ and $\tilde{h} = (\tilde{h}_1, \ldots, \tilde{h}_K) \in \mathcal{H}$, the independence of $g_i$ and equalities $\expected [g_i^2] = 1$
imply that
\begin{align*}
    \expected[(\mathfrak{X}_{h_c} - \mathfrak{X}_{\tilde{h}_c})^2] &= 
    \sum_{i=1}^m [\max\{h_{c1}(x_i), \ldots, h_{cK}(x_i)\} - 
    \max\{\tilde{h}_{c1}(x_i), \ldots, \tilde{h}_{cK}(x_i)\}]^2 \\\
        \expected[(\mathfrak{Y}_{h} - \mathfrak{Y}_{\tilde{h}})^2] &= 
    \sum_{i=1}^m [(h_{1}(x_i) - \tilde{h}_1(x_i))^2 + \cdots + (h_{K}(x_i) - \tilde{h}_K(x_i)^2] \,.
\end{align*}
For any $\mathbf{a} = (a_1, \ldots, a_K)$ and $\mathbf{b} = (b_1, \ldots, b_K) \in \mathbb{R}^K$, it can be directly checked that
\[
|\max\{a_1, \ldots, a_K\} - \max\{b_1, \ldots, b_K\}| 
\leq \max\{|a_1 - b_1|, \ldots, |a_K - b_K|\} 
\leq \sum_{i=1}^{K}|a_i - b_i| \, .
\]
Using the above inequality, we have the following bounds between $\mathfrak{X}_{h_c}$ and $\mathfrak{Y}_h$ for all $h, \tilde{h} \in \mathcal{H}$:
\begin{align*}
\expected [(\mathfrak{X}_{h_c} -\mathfrak{X}_{\tilde{h}_c})^2 ]
&\leq \sum_{i=1}^m \max\{
|h_{c1}(x_i) - \tilde{h}_{c1}(x_i)|, \ldots, |h_{cK}(x_i) - \tilde{h}_{cK}(x_i)|
\}^2\\
& = \sum_{i=1}^m \max\{
|h_{c1}(x_i) - \tilde{h}_{c1}(x_i)|^2, \ldots, |h_{cK}(x_i) - \tilde{h}_{cK}(x_i)|^2
\}\\
&\leq \sum_{i=1}^m \sum_{j=1}^{K}  |h_{cj}(x_i) - \tilde{h}_{cj}(x_i)|^2\\
&\leq \sum_{i=1}^m \sum_{j=1}^{K}  |h_{j}(x_i) - \tilde{h}_{j}(x_i)|^2 = 
\expected [(\mathfrak{Y}_h - \mathfrak{Y}_{\tilde{h}})^2] . 
\end{align*}
Finally, we can prove the lemma using this inequality and lemma A.1 of Lei et al. \cite{NIPS2015_3a029f04-multi-svm}.
\end{proof}
Note that the lemma holds if we substitute $h_j(x_i)$  with $h_j(x_i) + a_{ij}$, and $h_{cj}(x_i)$ with $h_{cj}(x_i) + a_{ij}$ in Eq. (\ref{eq:gaussian}), where $(a_{11}, \ldots, a_{mK})$ are constants. 

We substitute lemma 1 of Li et al. \cite{NEURIPS2018_1141938b-multiclass} with 
the following lemma.
\begin{lemma}
The empirical Rademacher complexity of $\mathcal{L}_c$ with example $S$ of size $m$ satisfies the following:
\[
R_{S}(\mathcal{L}_c) \leq \frac{\sqrt{2\pi}}{m}\expected_{\boldsymbol g}\sup_{h=(h_1,\ldots, h_K)\in \mathcal{H}_{p, \kappa}}
\sum_{i=1}^{m}\sum_{j=1}^{K} g_{(j-1)n+i} h_j(x_i) \,,
\]
where $g_1, \ldots, g_{nK}$ are independent random variables following Gaussian distribution $N(0, 1)$.
\end{lemma}
\begin{proof}
Following the proof of lemma 1 of Li et al. \cite{NEURIPS2018_1141938b-multiclass}, we have 
\begin{align*}
    R_{S}(\mathcal{L}_c) &\leq \frac{1}{m}\expected_{\boldsymbol \sigma}
    \left[
    \sup_{h_c \in \mathcal{H}_{p, \kappa, c}} \sum_{i=1}^m \sigma_i h_c(x_i, y_i)
    \right] 
    + \frac{1}{m}\expected_{\sigma} 
 \left[
 \sup_{h_c \in \mathcal{H}_{p, \kappa, c}} \sum_{i=1}^m \sigma_i 
 \max_{y \in \mathcal{Y}}(h_c(x_i, y) - \gamma \mathbf{1}_{y=y_i})
 \right]\\
 &\leq \frac{1}{m}\expected_{\boldsymbol \sigma}
    \left[
    \sup_{h \in \mathcal{H}_{p, \kappa, }} \sum_{i=1}^m \sigma_i h(x_i, y_i)
    \right] 
    + \frac{1}{m}\expected_{\sigma} 
 \left[
 \sup_{h_c \in \mathcal{H}_{p, \kappa, c}} \sum_{i=1}^m \sigma_i 
 \max_{y \in \mathcal{Y}}(h_c(x_i, y) - \gamma \mathbf{1}_{y=y_i}).
 \right]\\
\end{align*}
Following the proof, the first term can be bounded:
\begin{align*}
     \frac{1}{m}\expected_{\boldsymbol \sigma}
    \left[
    \sup_{h \in \mathcal{H}_{p, \kappa, }} \sum_{i=1}^m \sigma_i h_(x_i, y_i)
    \right]  \leq \sqrt{\frac{\pi}{2}} \mathfrak{G}_S(\mathcal{H}_{p, \kappa})
    \leq \frac{1}{m} \sqrt{\frac{\pi}{2}}\expected_{\boldsymbol g}\sup_{h=(h_1,\ldots, h_K)\in \mathcal{H}_{p, \kappa}}
\sum_{i=1}^{m}\sum_{j=1}^{K} g_{(j-1)n+i} h_j(x_i) \,.
\end{align*}
By using lemma~\ref{lemma:gaussian}, we can bound the second term:
\begin{align*}
&   \frac{1}{m}\expected_{\sigma} 
 \left[
 \sup_{h_c \in \mathcal{H}_{p, \kappa, c}} \sum_{i=1}^m \sigma_i 
 \max_{y \in \mathcal{Y}}(h_c(x_i, y) - \gamma \mathbf{1}_{y=y_i})
 \right] \\
& \leq \frac{1}{m} \sqrt{\frac{\pi}{2}}  \expected_{\boldsymbol g}
 [g_i \max(h_{c1}(x_i) - \gamma \mathbf{1}_{y_i = 1}, \ldots, 
 h_{cK}(x_i) - \gamma \mathbf{1}_{y_i = K}
 ]\\
& \leq  \frac{1}{m} \sqrt{\frac{\pi}{2}} \expected_{\boldsymbol g} \sup_{h=(h_1,\ldots, h_K)\in \mathcal{H}_{p, \kappa}}
 \sum_{i=1}^m \sum_{j=1}^K g_{(j-1)m+i}(h_j(x_i) - \gamma \mathbf{1}_{y_i = j})\\
&  =  \frac{1}{m} \sqrt{\frac{\pi}{2}} \expected_{\boldsymbol g} \sup_{h=(h_1,\ldots, h_K)\in \mathcal{H}_{p, \kappa}}
 \sum_{i=1}^m \sum_{j=1}^K g_{(j-1)m+i}h_j(x_i). 
\end{align*}
Adding these bounds can prove the lemma. 
\end{proof}
We use the above lemma instead of lemma 1 of Li et al. \cite{NEURIPS2018_1141938b-multiclass} for proving Theorem 1 Li et al. \cite{NEURIPS2018_1141938b-multiclass}, which 
gives a proof for our theorem.

\section{Proof of Theorem 6.4}

\begin{proof}
We first prove the bound for $L_{S, \rho}^{\mathrm{add}}(h)$. Following
the proof of Theorem 1 in Cortes et al.\cite{NIPS2016_535ab766-structure}, we can prove that
$L_{\mathcal{D}}(h_c) < L_{\mathcal{D}, \rho}^{\mathrm{add}}(h_c)$ and
\begin{align*}
     L_{\mathcal{D}, \rho}^{\mathrm{add}}(h_c) \leq L_{S, \rho}^{\mathrm{add}}(h_c) + 2R_{m}(\mathcal{H}_{c1}) + B\sqrt{\frac{\log \frac{1}{\delta}}{2m}} \,,
\end{align*}
where $\mathcal{H}_{c1}$ is defined:
\begin{align*}
&\mathcal{H}_{c1}\defeq  \left\{ 
(x, y) \mapsto \max_{y^\prime \neq y}\left(\mathsf{L}(y^\prime, y) -
\frac{\left(h(x, y) - h(x, y^\prime)\right)}{\rho}\right) : h_c \in \mathcal{H}_c
\right\}.
\end{align*}
We give a bound on the empirical Rademacher complexity of $\mathcal{H}_{c1}$. Due to the sub-additivity of the supremum, the following holds:
\begin{align*}
     R_{S}(\mathcal{H}_{c1}) \leq \frac{1}{m}\expected_{\boldsymbol \sigma}\left[
    \sup_{h_c \in \mathcal{H}_c}\sum_{i=1}^{m}\sigma_i \max_{y^\prime \neq y_i}
    \left(\mathsf{L}(y^\prime, y_i) + \frac{h_c(x_i, y^\prime)}{\rho}\right)
    \right] 
+       \frac{1}{m}\expected_{\boldsymbol \sigma}\left[
    \sup_{h_c \in \mathcal{H}_c}\sum_{i=1}^{m}\sigma_i  \frac{h_c(x_i, y_i)}{\rho}
    \right] \,.
\end{align*}
We first bound the first term with the Lipschitzness of $h \mapsto \max_{y^\prime \neq y}
\left(\mathsf{L}(y^\prime, y_i) + \frac{h_c(x_i, y^\prime)}{\rho} \right)$ for any requirement $c$.
For any $h, \tilde{h} \in \mathcal{H}$,
\begin{align*}
 \left|\max_{y \neq y_i}
\left(\mathsf{L}(y, y_i) + \frac{h_c(x_i, y)}{\rho} \right)
 -  \max_{y  \neq y_i}
\left(\mathsf{L}(y, y_i) + \frac{\tilde{h}_c(x_i, y)}{\rho} \right) \right| 
& \leq \frac{1}{\rho}\max_{y \neq y_i} 
\left| h_c(x_i, y) - \tilde{h}_c(x_i, y) \right|\\
& \leq  \frac{1}{\rho}\max_{y \in \mathcal{Y}} 
\left| h_c(x_i, y) - \tilde{h}_c(x_i, y) \right|\\
& \leq  \frac{1}{\rho}\max_{y \in \mathcal{Y}} 
\left| h(x_i, y) - \tilde{h}(x_i, y) \right| \,,
\end{align*}
where we use the fact that
$|h_c(x, y) - \tilde{h}_c(x, y)| \leq |h(x, y) - \tilde{h}(x, y)|$ for any 
$h, \tilde{h}\in \mathcal{H}$ and $(x, y) \in \mathcal{X} \times \mathcal{Y}$ since
$|h_c(x, y) - \tilde{h}_c(x, y)| = 0$ or equals  $|h(x, y) - \tilde{h}(x, y)|$, depending
on requirement $c$.
Following the proof of Theorem 1 in Cortes et al. \cite{NIPS2016_535ab766-structure}, we have
\begin{align*}
&   \frac{1}{\rho}\max_{y \in \mathcal{Y}} \left| h(x_i, y) - \tilde{h}(x_i, y) \right| \\
&\leq \frac{\sqrt{|F_i|}}{\rho} \sqrt{\sum_{f \in F_i}\sum_{y \in \mathcal{Y}_f} \left| h_f(x_i, y) - \tilde{h}_f (x_i, y)\right|^2} \,.
\end{align*}
We can apply Lemma 5 of Cortes et al. \cite{NIPS2016_535ab766-structure}, which yields:
\begin{align*}
&\frac{1}{m}\expected_{\boldsymbol \sigma}\left[
    \sup_{h_c \in \mathcal{H}_c}\sum_{i=1}^{m}\sigma_i \max_{y^\prime \neq y_i}
    \left(\mathsf{L}(y^\prime, y_i) + \frac{h_c(x_i, y^\prime)}{\rho}\right)
    \right] \\
&\leq \frac{\sqrt{2}}{m} \expected\left[
\sup_{h \in \mathcal{H}}\sum_{i=1}^m\sum_{f \in F_i}\sum_{y \in \mathcal{Y}_f}
\epsilon_{i, f, y} \frac{\sqrt{|F_i|}}{\rho} h_f(x_i, y) 
\right]  \\
&= \frac{\sqrt{2}}{\rho} R_{S}^{G}(\mathcal{H}) \,.
\end{align*}
Similarly, for the second term, the following Lipschitz property holds:
\begin{align*}
\left|\frac{h_c(x_i, y_i)}{\rho} - \frac{\tilde{h}_c(x_i, y_i)}{\rho}\right| &\leq  \frac{1}{\rho}\max_{y \in \mathcal{Y}} 
\left| h_c(x_i, y) - \tilde{h}_c(x_i, y) \right|\\
& \leq  \frac{1}{\rho}\max_{y \in \mathcal{Y}} 
\left| h(x_i, y) - \tilde{h}(x_i, y) \right|\\
& \leq  \frac{1}{\rho}\max_{y \in \mathcal{Y}} 
\left| h(x_i, y) - \tilde{h}(x_i, y) \right|\\
&\leq \frac{\sqrt{|F_i|}}{\rho} \sqrt{\sum_{f \in F_i}\sum_{y \in \mathcal{Y}_f} \left| h_f(x_i, y) - \tilde{h}_f (x_i, y)\right|^2} \,.
\end{align*}
Therefore, we can also obtain bound
\begin{align*}
  \frac{1}{m}\expected_{\boldsymbol \sigma}\left[
    \sup_{h_c \in \mathcal{H}_c}\sum_{i=1}^{m}\sigma_i  \frac{h_c(x_i, y_i)}{\rho}
    \right] \leq  \frac{\sqrt{2}}{\rho} R_{S}^{G}(\mathcal{H}) \, . 
\end{align*}
Taking the expectation over $S$ of the two inequalities shows that $R_{m}(\mathcal{H}_{c1}) \leq \frac{2\sqrt{2}}{\rho}R_{m}^{G}(\mathcal{H})$,
which completes the proof of the first statement. 

For the second statement, we follow a proof of Cortes et al. \cite{NIPS2016_535ab766-structure} to obtain:
\[
L_{\mathcal{D}, \rho}^{\mathrm{mult}}(h_c) \leq L_{S, \rho}^{\mathrm{mult}}(h_c) + 2R_{m}(\tilde{\mathcal{H}}_{c1}) + B\sqrt{\frac{\log \frac{1}{\delta}}{2m}} \,,
\]
where
\[
\tilde{\mathcal{H}}_{c1} \defeq \left\{ 
(x, y) \mapsto \max_{y^\prime \neq y} \mathsf{L}(y^\prime, y) \left( 
1 - \frac{h_c(x, y) - h_c(x, y^\prime)}{\rho} 
\right)
: h_c \in \mathcal{H}_c
\right\} \,.
\]
We can see the following inequality holds:
\begin{align*}
\left|
    \max_{y^\prime \neq y_i} \mathsf{L}(y^\prime, y_i) \left( 
1 - \frac{h_c(x_i, y_i) - h_c(x_i, y^\prime)}{\rho} 
\right) - \right.& \left.
    \max_{y^\prime \neq y_i} \mathsf{L}(y^\prime, y_i) \left( 
1 - \frac{\tilde{h}_c(x_i, y_i) - \tilde{h}_c(x_i, y^\prime)}{\rho} 
\right)  \right|\\
&\leq  \frac{2B}{\rho}\max_{y \in \mathcal{Y}} 
\left|
h_c(x_i, y) - \tilde{h}_c(x_i, y)
\right|\\
& \leq  
\frac{2B}{\rho}\max_{y \in \mathcal{Y}} 
\left|
h(x_i, y) - \tilde{h}(x_i, y)
\right| \,,
\end{align*}
where we use the fact  $|h_c(x, y) - \tilde{h}_c(x, y)| \leq |h(x, y) - \tilde{h}(x, y)|$ for any 
$h, \tilde{h}\in \mathcal{H}$ and $(x, y) \in \mathcal{X} \times \mathcal{Y}$. The reminder of the proof is identical as in the
previous argument.
\end{proof}

\section{Summary of Notations}
Table \ref{tab:my_label} shows the notations used in the paper.
\begin{table}[h]
    \centering

    \begin{tabular}{ll}
    \toprule
        Symbol & Meaning \\\midrule
         $\mathcal{X}$& domain of inputs \\
        $\mathcal{Y}$ & domain of labels \\
        $Z $& domain of examples \\
        $h$& a hypothesis \\
        $\mathcal{H}$ & a hypothesis class\\
                $c: \mathcal{X} \times \mathcal{Y} \to \{0, 1\} $ & requirement function \\
        $h_c$ & hypothesis $h$ modified to satisfy requirement $c$ \\
        $\mathcal{H}_c$ & set of modified hypotheses defined  as $\mathcal{H}_c := \{h_c : h \in \mathcal{H}\}$\\
        $\ell : \mathcal{H} \times Z \to \mathbb{R}_{+}$ & loss function \\
        $\ell_{\operatorname{0-1}}$ & the $\operatorname{0-1}$ loss function \\
        $\ell_{\rho}$ & a  margin loss function \\
        $S = z_1, \ldots, z_m$ & a sequence of $m$ examples \\
        $\mathcal{D}$ & a distribution over $\mathcal{X} \times \mathcal{Y}$ \\
        $L_\mathcal{D}(h)$ & generalization error of $h$ \\
        $L_\mathcal{S}(h)$ & empirical error of $h$  over $S$\\
        $L_\mathcal{S, \rho}(h)$ & empirical margin error of $h$ over $S$\\
        $R_{S}(\mathcal{G})$ &the empirical Rademacher complexity of $\mathcal{G}$ with respect to $S$ \\
                $R_{m}(\mathcal{G})$ &the  Rademacher complexity of $\mathcal{G}$ \\
        $R_{m}(\mathcal{G}; r)$ &the empirical local Rademacher complexity of $\mathcal{G}$ \\
        $R^{G}_{m}(\mathcal{H})$ &the empirical factor graph Rademacher complexity of hypothesis class $\mathcal{H}$ \\
        $\rho_h(x, y)$ & a margin function \\
        
            \bottomrule
    \end{tabular}
    \caption{Summary of notations}
    \label{tab:my_label}
\end{table}

\end{document}